\documentclass[]{fairmeta}
\usepackage{amsthm}
\usepackage{amssymb}
\newtheorem{assumption}{Assumption}
\newtheorem{lemma}{Lemma}

\title{DUET — Dual User Embedding Transformers for Offsite Conversion Prediction}

\author[\dagger]{Reazul Hasan Russel}
\author[\dagger]{Mingwei Tang}
\author[\dagger]{Rostam Shirani}

\author[]{Xinlong Liu}
\author[]{Navid Madani}
\author[]{Leo Ding}
\author[]{Yawen He}
\author[]{Xiangyu Wang}
\author[]{Mustafa Acar}
\author[]{Ashish Katiyar}
\author[]{Yuhai Li}
\author[]{Alan Yang}
\author[]{Metarya Ruparel}
\author[]{Derek Qiang Xu}
\author[]{Rupert Wu}
\author[]{Rui Yang}
\author[]{Liang Tao}
\author[]{Xinyi Zhao}
\author[]{Larry Zhang}
\author[]{Sri Reddy}
\author[]{Rob Malkin}

\contribution[\dagger]{Equal contribution}
\contribution[]{AI at Meta}

 \abstract{Offsite conversion rate (OCVR) prediction is an important ranking problem in computational recommendation systems. This task presents a modeling challenge: click signals are abundant and exhibit short temporal horizons, whereas conversion signals are inherently sparse, long-delayed, and frequently unattributed. Despite these statistical disparities, both signal types must inform models that operate within strict serving-latency constraints. Prior pre-training approaches address this heterogeneity with a  single, undifferentiated encoder applied uniformly across both data streams. We propose DUET (Dual User Embedding Transformers), a framework that explicitly partitions user behavioral data into two domain-coherent streams — clicks and conversions — and pre-trains dedicated transformer encoders with architectures tailored to each stream's statistical characteristics: multi-layer self-attention for the dense click stream and interleaved cross- and self-attention for the sparse conversion stream. The resulting complementary embeddings are jointly consumed by a downstream ranker without exceeding serving-latency budgets. Evaluation demonstrates up to 0.38\% normalized entropy (NE) reduction relative to the strongest baseline, and A/B test shows consistent improvements in OCVR prediction accuracy.}

\correspondence{Reazul Hasan Russel at \email{rrussel@meta.com}}

\begin{document}

\maketitle

  \section{Introduction}
  \label{sec:intro}                                                                                        
  \emph{Offsite Conversion (OCVR) prediction} refers to the task of estimating the probability that a user completes a target action—such as a purchase or registration—on an external destination (website or application) after engaging with a recommended item served on the host platform. The importance of OCVR prediction is growing, driven by the expansion of retail media networks and the deprecation of third-party cookies, which together are shifting performance budgets toward platforms with first-party behavioral data. Offsite retail media is forecast to grow annually, making accurate OCVR prediction a differentiator for recommendation platforms. Yet the OCVR prediction task is difficult: positive labels are sparse (rates below 5\%), attribution windows are long and variable (hours to days), and a fraction of conversions remain unattributed depending on product segment and attribution method. Recommendation models must contend with all of these challenges under stringent online training and inference latency constraints, creating a persistent tension between prediction quality and serving efficiency.
                 
  A natural response to these constraints is to focus on richer user representations --- models that capture deeper behavioral patterns should yield better conversion predictions. Indeed, recent advances in transformer-based sequence modeling~\cite{kang2018sasrec, xiong2026latte}, self-supervised
  pre-training~\cite{ouyang2023contrastive, Shtoff_2023, Zhang2023ContrastiveLW}, and embedding methods~\cite{su2025multifacetedlargeembeddingtables, Wang_2024} have demonstrated quality gains in recommendation and ranking tasks. However, the computational cost of these expressive architectures makes direct deployment in the latency-critical serving path impractical. This motivates a decoupled design: an upstream model pre-trains rich user embeddings offline, which are then served asynchronously as static embedding features to the ranker. The separation allows expressive architectures to inform ranking models without violating latency budgets. Existing instantiations of this paradigm~\cite{tang2024afl} have validated the approach for click prediction, but they train a single upstream model on click and onsite conversion data, treating all behavioral signals as a homogeneous stream. This design introduces three limitations when the goal is OCVR prediction:
  \begin{enumerate}
      \item \textbf{Signal dominance under regime mismatch.} Click and onsite conversion data is dense with short attribution windows, while offsite conversion data is orders of magnitude sparser. A single model trained on both is dominated by the denser click signal, underrepresenting the conversion patterns most relevant to the downstream task.            
      \item \textbf{Architectural uniformity.} Applying a single encoder architecture to data streams with fundamentally different statistical properties ignores the possibility that dense click sequences and sparse conversion sequences benefit from different inductive biases --- an observation we validate empirically in Section~\ref{sec:architecture}.                                              
      \item \textbf{Shallow cross-domain transfer.} Knowledge transfer across applications has typically been limited to shared feature encoders~\cite{fu2024unifiedframeworkmultidomainctr, tian2023ufinuniversalfeatureinteraction} or domain-disentangled representations~\cite{Zhu_2023, Wang_2025}, without incorporating organic engagement or content-derived semantic signals that can enrich user representations.                                                                         
  \end{enumerate}                                                   

  We propose \textbf{DUET} (Dual User Embedding Transformers), a framework that aims to address these limitations through a single, unified principle: \emph{learn dedicated user embeddings from statistically coherent data streams, then compose them in the downstream ranker}. Concretely, DUET partitions upstream training data into two domain-coherent streams --- a click/onsite-conversion stream and an offsite-conversion stream --- and pre-trains a dedicated transformer encoder on each, with an architecture tailored to the stream's statistical properties. The resulting complementary embeddings, which we term \textbf{ClickAUE} (Click Attentive User Embedding) and \textbf{ConvAUE} (Conversion Attentive User Embedding), are jointly consumed by the downstream ranker for offsite OCVR prediction.                                                                                  
  The main contributions of this work are:                                                                 
  \begin{enumerate}                                                                                        
    \item \textbf{Domain-specialized dual embedding learning.} We introduce domain-aware data routing that partitions heterogeneous behavioral data into click and conversion streams, and pre-train a dedicated upstream encoder on each. We select attention architectures matched to each stream's regime: multi-layer self-attention (LLaTTE~\cite{xiong2026latte}) for the dense click data stream and interleaved cross-/self-attention for the sparse conversion data stream. Ablations confirm that each architecture choice is load-bearing for its respective stream (Section~\ref{sec:architecture}).
    \item \textbf{Multimodal, cross-application input enrichment.} Both upstream encoders consume event-based features (EBFs) spanning content interactions, organic feed engagement, and content-derived semantic IDs~\cite{roychowdhury2026coffeecodesignframeworkfeature} across multiple applications, enabling richer user representations.                           
    \item \textbf{Asynchronous serving with scalable infrastructure.} An event-triggered inference mechanism generates user embeddings asynchronously, decoupling upstream model complexity from the serving-latency budget. The deployment incorporates throttling, checkpoint validation, and embedding quantization, achieving negligible training QPS and serving latency overhead .
    \item \textbf{Empirical validation.} We evaluate DUET across $6$ downstream OCVR models, demonstrating meaningful metric improvements.         
  \end{enumerate}

\section{Related Work}
\label{sec:related}
  
We review prior work along three axes most relevant to our contributions.       

  \textbf{CVR Prediction} General conversion rate (CVR) prediction is challenged by label sparsity and sample selection bias. ESMM~\cite{ma2018esmm} jointly models CTR and CVR to estimate conversion rate over the full impression space, and subsequent work extends this with counterfactual multi-task objectives (ESCM$^2$)~\cite{Wang_2022}, doubly robust estimation~\cite{dai2022generalizeddoublyrobustlearning, li2023tdrcltargeteddoublyrobust}, and variational information exploitation~\cite{fei2025esvirec}. Multi-task architectures such as MMoE~\cite{ma2018mmoe} and PLE~\cite{zhou2023featuredecompositionreducingnegative} balance shared and task-specific representations within a single model. These approaches address sparsity through task decomposition or debiasing at the model level. Our work is complementary: we address sparsity upstream by routing data into domain-coherent streams and pre-training dedicated encoders whose embeddings are consumed by any downstream CVR architecture. 
                                                            
  The challenge of unattributed and offsite conversions --- where events cannot be deterministically matched to impressions due to cross-device fragmentation, cookie restrictions, or app tracking policies -- has grown over time. Entire-space counterfactual methods~\cite{Wang_2022} and doubly robust estimators~\cite{dai2022generalizeddoublyrobustlearning} partially address this, but remain limited when attribution rates are structurally low. Our framework incorporates unattributed conversion events directly into the ConvAUE training stream, using them as weak supervision to mitigate sparsity rather than discarding them.
  
  \textbf{Sequence Modeling for Recommendation}                                                        
  Modeling user behavior as ordered interaction sequences has driven advances in recommendation. SASRec~\cite{kang2018sasrec} introduced self-attention for sequential recommendation; target-aware variants such as DIN~\cite{zhou2018din} and DIEN~\cite{zhou2019dien} activate the most relevant historical behaviors; and multi-interest architectures~\cite{li2019mind, xiao2020dmin} represent users as mixtures of latent interest vectors. Scaling laws for sequential models~\cite{zhang2023scalinglawlargesequential, zhang2024wukongscalinglawlargescale} establish capacity--performance relationships that motivate deeper architectures. Most relevant to our work, LLaTTE~\cite{xiong2026latte} applies stackable transformer layers for sequence modeling with predictable power-law scaling in depth and sequence length. We adopt LLaTTE as the encoder for the dense click stream (ClickAUE) and apply interleaved cross-/self-attention for the sparse conversion stream (ConvAUE), matching the architecture to each stream's statistical properties.
  
  \textbf{Cross-Domain and Multi-Scenario Recommendation}
  Cross-domain recommendation integrates user signals across multiple apps or surfaces. Representative approaches include shared-plus-specific topologies (STAR~\cite{sheng2021star}), personalized parameter generation (PEPNet~\cite{chang2023pepnetparameterembeddingpersonalized}), adaptive domain-specific scaling (ADS~\cite{chai2025ads}), and contrastive alignment across heterogeneous scenarios~\cite{xie2022contrastivecrossdomainrecommendationmatching, Zitao_2023RecSys}. These methods transfer knowledge via shared encoders or representation alignment within a single model. Our framework differs by transferring knowledge through pre-trained embeddings: both ClickAUE and ConvAUE consume cross-application event-based features spanning ads interactions, organic engagement, and content-derived signals, but the resulting embeddings are composed downstream rather than aligned during training.
   
  \textbf{Pre-training and Asynchronous Serving for Ads}
  A persistent challenge in industrial recommendation is that high-capacity models improve quality but cannot satisfy online latency budgets. Two-tower architectures~\cite{covington2016youtube} and model compression~\cite{deng2021deeplightdeeplightweightfeature} address this at the model level. Self-supervised and contrastive pre-training offer an alternative: pre-training on unlabeled impressions~\cite{Shtoff_2023}, contrastive views of user--ad interactions~\cite{ouyang2023contrastive}, and masked multi-domain objectives~\cite{ouyang2024masked} learn transferable representations before task-specific fine-tuning. On the infrastructure side, multi-faceted large embedding tables~\cite{su2025multifacetedlargeembeddingtables} and the COFFEE codesign framework~\cite{roychowdhury2026coffeecodesignframeworkfeature} scale embedding representations for ranking.
  
  Most directly relevant to our work, asynchronous pre-training paradigms~\cite{tang2024afl} decouple representation learning from online inference by pre-training upstream user embeddings that are consumed as static features by lightweight rankers. SUM~\cite{zhang2024sum} demonstrates modular       upstream--downstream user modeling serving hundreds of downstream models, and LONGER~\cite{chai2025longer} extends this to long-horizon multi-scenario settings. Our framework builds on this paradigm but introduces two main departures: domain-aware data routing that partitions behavioral data into statistically coherent streams, and stream-specific encoder architectures that produce complementary rather than monolithic user representations.

    \begin{figure*}
      \centering
      \includegraphics[scale=0.35]{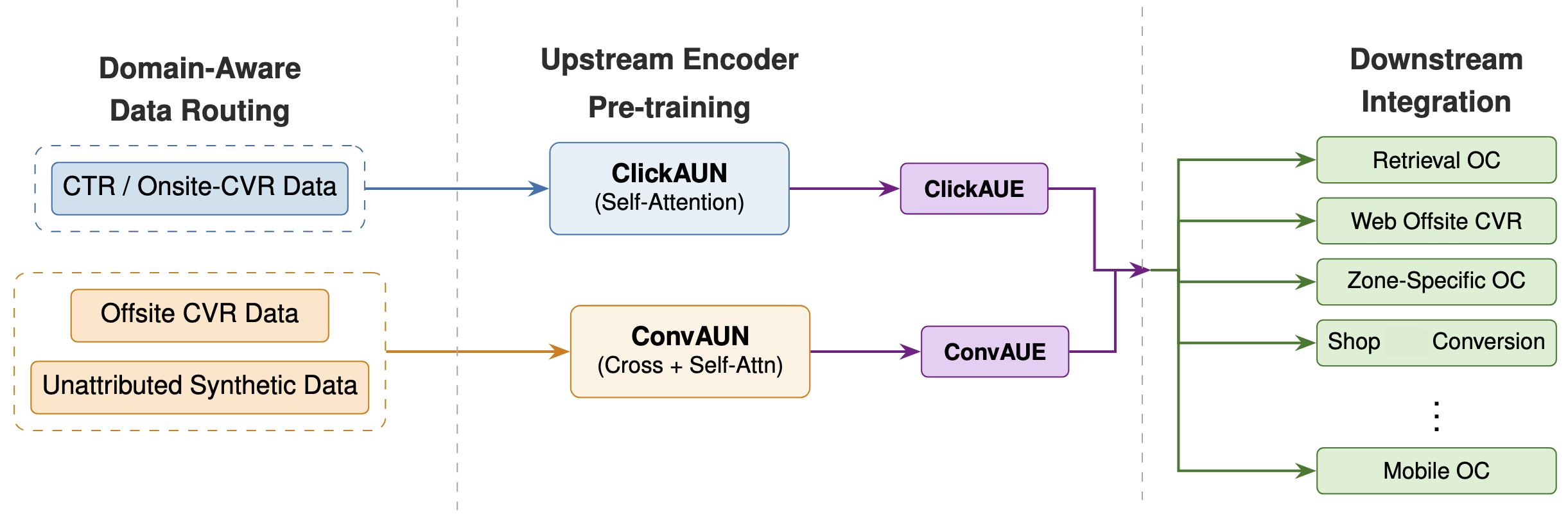}
      \caption{DUET unified framework: data routing, representation pre-training, and downstream ranker.}
      \label{fig:three_phases}
  \end{figure*}
  \section{DUET Framework}
  \label{sec:method}                                                                             
  
  Ranking models aim to predict the probability of click or conversion given a user-item pair. DUET decomposes this problem into three stages: partitioning training data into domain-coherent streams (Section~\ref{sec:routing}), pre-training a dedicated upstream encoder on each stream (Section~\ref{sec:encoder}), and integrating the resulting embeddings into the downstream ranker (Section~\ref{sec:downstream}).
  
  \subsection{Domain-Aware Data Routing}
  \label{sec:routing}                                
  We partition pointwise training data into two streams based on label semantics and attribution horizon.                         

  \paragraph{CTR / Onsite-CVR Stream.}
  This stream contains click and onsite conversion events, where the labeled action typically occurs within a short amount of time from the recommended item interaction. On this data stream, negative samples are downsampled and all positives are retained.  
  
  \paragraph{OCVR Stream.}               
  This stream contains user actions on external websites or third party apps, with a longer attribution window. We don't perform downsampling on OCVR training samples in order to preserve training data volume. 
  To mitigate the data sparsity issue in OCVR, the stream incorporates \emph{Synthetic data from unattributed conversion} --- conversion events that cannot be deterministically linked to a specific item exposure due to cross-device journeys, cookie restrictions, or delayed attribution. Synthetic data is generated by inferring the best possible ranking results associated with an unattributed conversion. Although individually noisy, these events carry aggregate user-level intent signals and increase the effective training volume.

  The statistical contrast between the two streams --- dense, short-horizon click data versus sparse, long-horizon conversion data --- motivates both the data separation and the architectural choices described next.

  \begin{figure*}
    \centering                                              
    \includegraphics[scale=0.32]{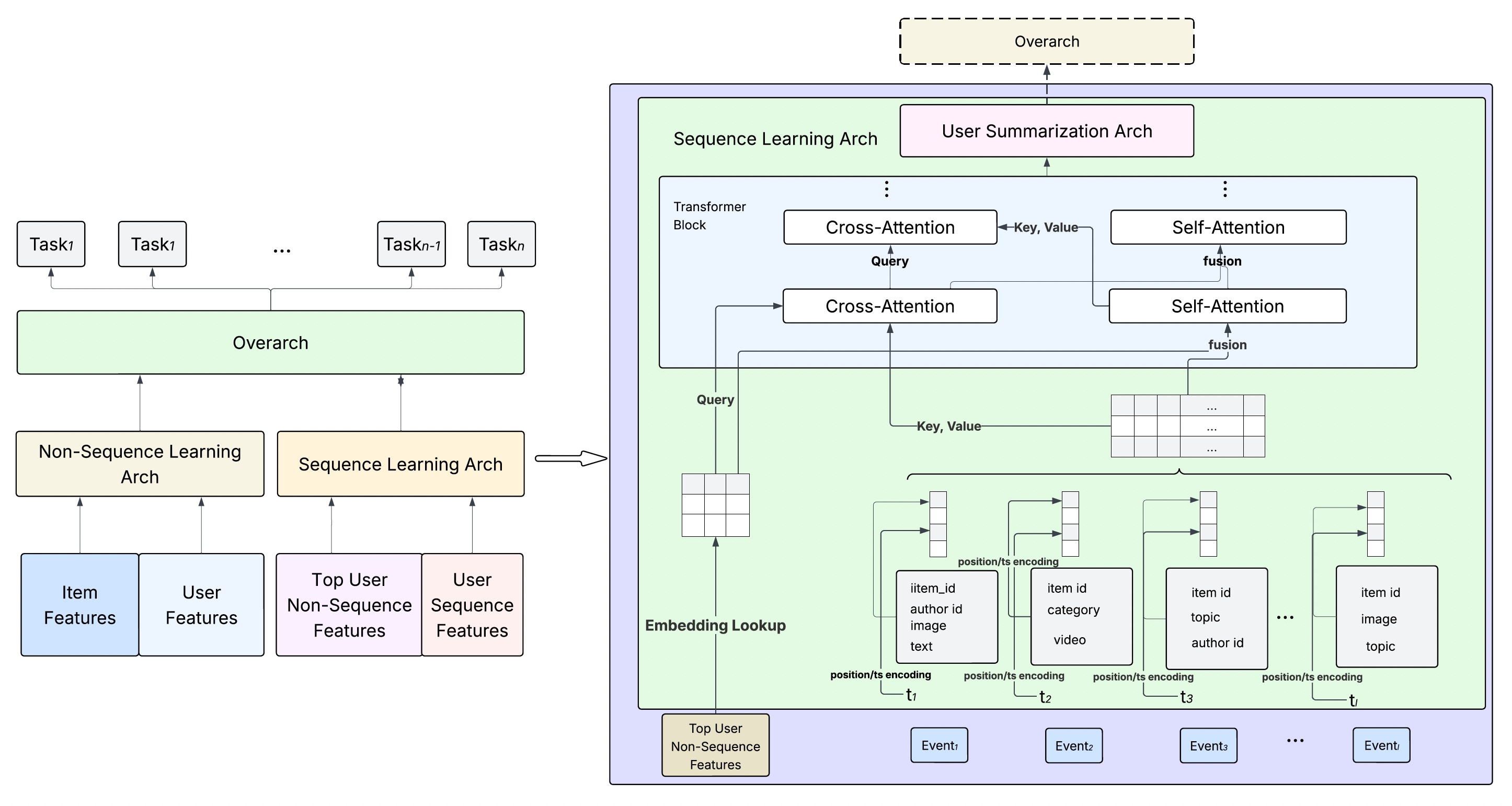}
    \caption{DUET architecture. Left: the DLRM backbone with parallel sequence and non-sequence branches. Right: event tower structure with stacked attention blocks and DCN user summarization.}
    \label{fig:sequence_arch}       
  \end{figure*} 
  
  \subsection{Upstream Encoder Design}
  \label{sec:encoder}     
  We instantiate two upstream encoders from a shared backbone architecture, each trained on one data stream. The encoders differ in attention configuration and output dimensionality but share the same input representation and backbone structure.                        

  \subsubsection{Input Representation}                      
  \label{sec:features}     
  Each encoder consumes user-side and target item-side features. Target item features are processed through the non-sequence branch of a DLRM backbone. User-side features consist of \emph{event-based feature (EBF) sequences} (~\cite{roychowdhury2026coffeecodesignframeworkfeature}) constructed from user engagement across recommended items (impressions, onsite conversions), organic feed (views, likes), and video content over a lookback window of several months, with each event represented by:  
  \begin{enumerate}
      \item \textbf{Timestamp}: a dense encoding of event occurrence time.
      \item \textbf{ID-based attributes}: entity-level features including item ID, author ID, media type, and position.
      \item \textbf{Semantic IDs}: compact discrete codes derived from entity content (image, text, video) via KNN~\cite{roychowdhury2026coffeecodesignframeworkfeature} or RQ VAE~\cite{ramasamy2025sidesemanticidembedding}, capturing content semantics beyond ID-based representations.                                                                                     
  \end{enumerate}                                                                                          
  
  \subsubsection{Backbone Architecture}
  \label{sec:backbone}     
  Both encoders follow a DLRM structure (Figure~\ref{fig:sequence_arch}, left) with two parallel branches: a \emph{sequence branch} that processes user EBF sequences and top user static features (selected by feature importance on the downstream task), and a \emph{non-sequence branch} that consumes ranking item-side and user non-sequence features. Branch outputs are fused in an overarch interaction layer~\cite{zhang2022dhendeephierarchicalensemble} for joint prediction during pre-training.

  Within the sequence branch, each event type is processed by a dedicated \emph{event tower} composed of transformer blocks (detailed below). The outputs of all event towers are concatenated and passed through a Deep \& Cross Network (DCN)~\cite{wang2017deepcrossnetwork, Wang_2021} user summarization module, which produces the final user embedding.

  We instantiate two models from this backbone:                               
  \begin{itemize}
      \item \textbf{ClickAUN} (Click-Attentive Upstream Network): trained on the CTR/onsite-CVR stream. Produces \emph{ClickAUE} user embeddings. 
      \item \textbf{ConvAUN} (Conversion-Attentive Upstream Network): trained on the OCVR stream. Produces \emph{ConvAUE} user embeddings.                   
  \end{itemize}                                             

  \subsubsection{Attention Configuration}                   
  \label{sec:architecture}
  Each event tower applies transformer blocks over sequence embeddings $X_{\text{sq}}$ and user static feature embeddings $X_{\text{st}}$. 
  We define two block types:     
  \paragraph{Self-Attention Block.}                         
  Static and sequence embeddings are concatenated as 
  \begin{equation}
    X = \mathrm{concat}(X_{\text{st}}, X_{\text{sq}}),
  \end{equation}
  which is jointly processed as:
  \begin{equation}         
    Y_{\text{self}} = X + \mathrm{attention}(Q{=}\mathrm{LayerNorm}(X),
    K{=}\mathrm{LayerNorm}(X), V{=}\mathrm{LayerNorm}(X))
    \label{eq:self_attn_1},
  \end{equation}
  where $Q$, $K$, and $V$ are the query, key, and value as inputs of $attention$, respectively. This allows the model to capture complex interactions between static and sequence feature types. The output of the self-attention block is      
  \begin{equation} 
    \mathrm{SelfAttn}(X_{\text{st}}, X_{\text{sq}}) = Y_{\text{self}} + \mathrm{FFN}(\mathrm{LayerNorm}(Y_{\text{self}}))         
    \label{eq:self_attn_2}.                       
  \end{equation}       
  \paragraph{Cross-Attention Block.}                        
  Static embeddings query against sequence embeddings as:
  \begin{equation}
    Y_{\text{cross}} = X_{\text{st}} + \mathrm{attention}(Q{=}\mathrm{LayerNorm}(X_{\text{st}}),\\ K{=}\mathrm{LayerNorm}(X_{\text{sq}}),
  V{=}\mathrm{LayerNorm}(X_{\text{sq}}))                                                                          
    \label{eq:cross_attn_1}.                                 
  \end{equation}
  This enables the model to contextualize sequential user behaviors using high-importance static features. The output of the cross-attention block is
  \begin{equation}
    \mathrm{CrossAttn}(X_{\text{st}}, X_{\text{sq}}) = Y_{\text{cross}} + \mathrm{FFN}(\mathrm{LayerNorm}(Y_{\text{cross}}))
    \label{eq:cross_attn_2}.
  \end{equation}     
  \paragraph{Stream-Specific Configuration.}                                                               
  ClickAUN stacks multiple self-attention layers following the LLaTTE paradigm~\cite{xiong2026latte}: dense supervision and short attribution windows make deep self-attention effective for capturing high-order interaction patterns. ConvAUN interleaves cross-attention and self-attention blocks. Under sparse positive labels, pure self-attention over long sequences risks overfitting to uninformative majority-negative patterns; cross-attention against stable user-level attributes anchors sequence representations, providing implicit regularization. This design also reduces computational cost, as the cross-attention query length is bounded by the number of static features rather than the full sequence length. We validate these configuration choices through ablations in Section~\ref{sec:ablation}.

  \subsubsection{Training Objectives}                       
  \label{sec:objectives}  
  Let $\ell_{\mathrm{BCE}}(\hat{y},\, y) = -\bigl[y \log \hat{y} + (1 - y)\log(1 - \hat{y})\bigr]$ denote the binary cross-entropy (BCE) loss for a predicted probability $\hat{y} \in (0,1)$ and binary label $y \in \{0,1\}$. Both upstream models are trained with multitask objectives that combine task-specific BCE losses.         
  ClickAUN is optimized on the CTR/onsite-CVR stream:       
  \begin{equation}                                          
    \mathcal{L}_{\mathrm{ClickAUN}}                             
      = \alpha_{1}\,\ell_{\mathrm{BCE}}(\hat{y}_{\mathrm{ctr}},\, y_{\mathrm{ctr}})
      + \alpha_{2}\,\ell_{\mathrm{BCE}}(\hat{y}_{\mathrm{onsite}},\, y_{\mathrm{onsite}}),
    \label{eq:loss_clickaun}
  \end{equation}                                            
  where $\hat{y}_{\mathrm{ctr}}$ and $\hat{y}_{\mathrm{onsite}}$ are the predicted click-through and onsite conversion probabilities, $y_{\mathrm{ctr}}, y_{\mathrm{onsite}} \in \{0,1\}$ are the corresponding labels, and $\alpha_{1}, \alpha_{2} > 0$ are task weights. The CTR task provides dense supervisory signal; the onsite       
  conversion task provides sparser but more intent-indicative supervision.
  ConvAUN is optimized on the offsite conversion stream:    
  \begin{equation}                                          
    \mathcal{L}_{\mathrm{ConvAUN}}
      = \beta_{1}\,\ell_{\mathrm{BCE}}(\hat{y}_{\mathrm{off}},\, y_{\mathrm{off}})
      + \beta_{2}\,\ell_{\mathrm{BCE}}(\hat{y}_{\mathrm{unattr}},\, y_{\mathrm{unattr}}).
    \label{eq:loss_convaun}                                     
  \end{equation}                                            
  where $\hat{y}_{\mathrm{off}}$ and $\hat{y}_{\mathrm{unattr}}$ are the predicted probabilities for attributed and unattributed offsite conversions, with labels $y_{\mathrm{off}}, y_{\mathrm{unattr}} \in \{0,1\}$ and task weights $\beta_{1}, \beta_{2} > 0$. All task weights are tuned on held-out validation data. 
  \subsection{Downstream Integration}                       
  \label{sec:downstream}                                    
  Let $\mathbf{e}_{\mathrm{click}}$ and $\mathbf{e}_{\mathrm{conv}}$ denote the ClickAUE and ConvAUE embeddings for a given user. The downstream ranker consumes these as additional input features alongside its standard feature vector~$\mathbf{x}$, with no other            
  architectural modifications. Both embeddings are frozen during downstream training---no gradients propagate back through $\mathbf{e}_{\mathrm{click}}$ or $\mathbf{e}_{\mathrm{conv}}$---enabling upstream and downstream models to retrain on independent schedules.  
  The downstream model is trained with a multitask objective:
  \begin{equation}
    \mathcal{L}_{\mathrm{down}}                                 
      = \lambda_{1}\,\ell_{\mathrm{BCE}}(\hat{y}_{\mathrm{cvr}},\, y_{\mathrm{cvr}})
      + \sum_{k=1}^{K} \lambda_{k+1}\,                              
        \ell_{\mathrm{BCE}}\!\bigl(\hat{y}_{\mathrm{aux}}^{(k)},\, y_{\mathrm{aux}}^{(k)}\bigr),
    \label{eq:loss_downstream}                                  
  \end{equation}                                            
  where the primary term is the offsite conversion loss, $\bigl\{\hat{y}_{\mathrm{aux}}^{(k)},\, y_{\mathrm{aux}}^{(k)}\bigr\}_{k=1}^{K}$ are predictions and labels for $K$~auxiliary tasks (e.g., value prediction, engagement prediction) that provide additional gradient signal for regularization, and $\{\lambda_{i}\}_{i=1}^{K+1}$  
  are the corresponding task weights. 
  
  
  \paragraph{Serving.}
  Embeddings are generated asynchronously via the Event-Triggered Inference (ETI) system (Section~\ref{sec:eti}), stored in a feature store, and retrieved at serving time within the latency budget.    
 \section{System Architecture}                               
  \label{sec:production} 

  \begin{figure*}[t]
    \centering
    \includegraphics[scale=0.35]{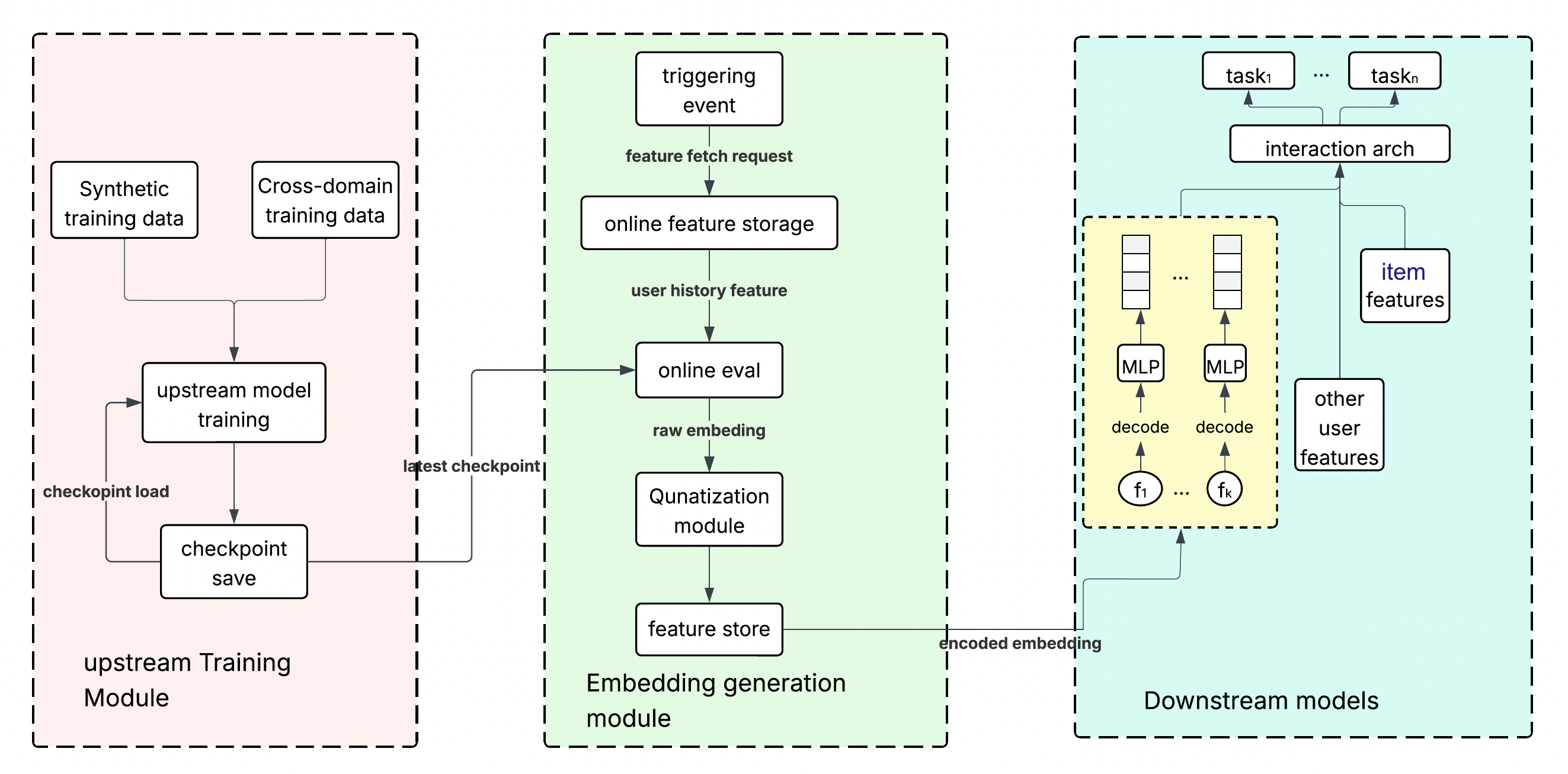}
    \caption{End-to-end system: recurring pre-training with checkpoint validation followed by SIDE-based embedding compression and finally event-triggered inference and serving.}
    \label{fig:serving_pipeline}                  
  \end{figure*}
  

  \subsection{Event-Triggered Inference}
  \label{sec:eti}                                                                                                               
  Conventional embedding pipelines couple embedding generation to training data ingestion, producing updates at the cadence of the training loop. This is problematic for DUET: the CTR/OCVR stream generates high data volume, while the OC stream is too sparse-- neither cadence yields a desirable freshness vs infra-load tradeoff.                                        
  
  We instead decouple embedding generation from training via an \emph{event-triggered inference} (ETI) architecture (as illustrated in Figure~\ref{fig:serving_pipeline}). While the model is doing pre-training on hourly/daily basis, a dedicated serving model, loading from the most recently validated pre-trained checkpoint, generates embeddings on demand: when a user performs a qualifying action (e.g., an onsite conversion), the system fetches the user's latest EBF and static features from online storage and runs a forward pass to produce an updated embedding.
  Onsite conversions provide a natural trigger --- their volume is moderate (higher than post-click events, lower than impressions) and correlates with user value, ensuring that high-activity users receive fresher representations. 

  \paragraph{Independent Pipeline Operation}               
  ClickAUN and ConvAUN pre-training 
  iterate at a cadence matched to its data volume --- ClickAUN retrains more frequently; ConvAUN is paced by sparser conversion data. Embedding staleness and serving health are monitored per-pipeline via automated alerts on NE degradation and latency.

  \paragraph{Checkpoint Validation}
  Training and inference are operationally separated: if a training run produces a degraded checkpoint (e.g., due to gradient instability), the serving model rejects any checkpoint whose validation NE exceeds a predefined threshold relative to its predecessor, maintaining embedding quality without manual intervention.

  \paragraph{Efficiency Optimizations}
  At inference time, only the sequence learning component is extracted from the training checkpoint, reducing model size and the number of serving hosts. A fixed throttling window suppresses redundant updates for users with frequent triggering events.

  \subsection{Embedding Compression}
  \label{sec:quantization}
  Raw embeddings from ETI are compressed by a standalone quantization module before storage (Figure~\ref{fig:serving_pipeline}, middle). We apply the Semantic ID Embedding (SIDE) technique~\cite{ramasamy2025sidesemanticidembedding}, which maps continuous embedding vectors to discrete codebook entries via vector quantization and fuses multiple quantized vectors into compact semantic ID representations. The quantization module is pre-trained offline on raw embeddings. SIDE achieves $4\times$ compression relative to FP16 storage with negligible downstream NE impact.

  \subsection{Serving-Time Decoding}
  \label{sec:serving}
  At serving time, quantized embeddings are fetched from the feature store and decoded via a rule-based decoder from SIDE into float-valued vectors. A learnable MLP, trained jointly with the downstream model, projects the decoded embeddings to match the dimensionality of other input features before they enter the overarch interaction layer. This design adds no architectural modifications to the downstream ranker beyond the additional embedding inputs.
\section{Experiments}
\label{sec:experiments}

We evaluate DUET through retrospective experiments, ablation studies, and A/B tests. Section~\ref{sec:metrics}--\ref{sec:impl_details} describe the evaluation metric, compared configurations, and implementation details. Section~\ref{sec:embedding_analysis} analyzes the learned embeddings. Sections~\ref{sec:main_results} and~\ref{sec:ablation} present the main results and component-level ablations. Finally section~\ref{sec:online} reports online results.

\subsection{Evaluation Metric}
\label{sec:metrics}

We use normalized entropy (NE) as the primary evaluation metric, defined as the average log loss normalized by the entropy of a na\"ive model predicting the empirical positive rate $p$:
\begin{equation}
  \text{NE} = \frac{-\frac{1}{N}\sum_{i=1}^{N}\big[y_i\log(\hat{p}_i) + (1{-}y_i)\log(1{-}\hat{p}_i)\big]}{-\big[p\log(p) + (1{-}p)\log(1{-}p)\big]}.
  \label{eq:ne}
\end{equation}
NE is preferred over raw log-loss because it normalizes for class imbalance, enabling meaningful comparison across tasks and datasets with different positive rates.
We report relative NE reduction (\%$\Delta$NE) against the baseline, where higher reduction is better. NE is used consistently across both upstream pre-training and downstream evaluation, enabling end-to-end performance tracking.

\subsection{Compared Configurations}
\label{sec:baselines}

  We evaluate four configurations, all sharing the same downstream DLRM architecture and input feature set: 
  \begin{enumerate}
      \item \textbf{Baseline}: the existing downstream ranker without pre-trained embeddings.
      \item \textbf{ClickAUE Only}: downstream ranker augmented with ClickAUE.
      \item \textbf{ConvAUE Only}: downstream ranker augmented with ConvAUE.
      \item \textbf{DUET}: downstream ranker augmented with both ClickAUE and ConvAUE.
  \end{enumerate}
  Contrastive and masked pre-training methods~\cite{ouyang2023contrastive, ouyang2024masked} are excluded as their augmentation-based objectives differ from our pointwise BCE formulation. The single-embedding configurations serve as ablations isolating each stream's marginal contribution.     


\subsection{Implementation Details}
\label{sec:impl_details}

Both ClickAUN and ConvAUN are trained on $128$ NVIDIA H100 GPUs, achieving approximately $\approx$200K offline training QPS—sufficient throughput to enable continuous training and timely deployment of updated model checkpoints under strict latency constraints. 

ClickAUN employs a multi-task architecture with seven task-specific towers, including CTR, onsite conversion, and video view prediction. Each task head's contribution is modulated via a gradient scaling hyperparameter, calibrated to balance the supervisory signal strength against task-specific label noise. The sequence learning architecture of ClickAUN comprises seven event towers, each corresponding to a distinct user-side event feature, with a maximum sequence length of $1,000$. The self-attention module consists of $n=2$ stacked layers, each with $h=2$ attention heads, model dimension $d_{model}=256$, and feed-forward dimension $d_{ff}=1024$. The output ClickAUE are $16$ $80$-dimension embeddings.

ConvAUN comprises seven task specific towers, including conversion optimized, link click, button click and synthetic conversion task. The sequence learning architecture adopts horizontal scale paradigm comprising $13$ event towers, where maximum event length not exceed $200$. The attention module consists of $n=1$ layer with $h=8$ attention heads, model dimension $d_{model}=128$, and feed-forward dimension $d_{ff}=256$. The output dimension of ConvAUE is $5\times 80$.
We use the existing downstream models as baselines, maintaining identical architecture and training data; only the input feature set is modified to include ConvAUE and ClickAUE as additional embedding features.

\begin{figure*}[ht] 
    \centering
    \begin{minipage}[t]{0.48\textwidth} 
        \centering
        \includegraphics[width=\linewidth]{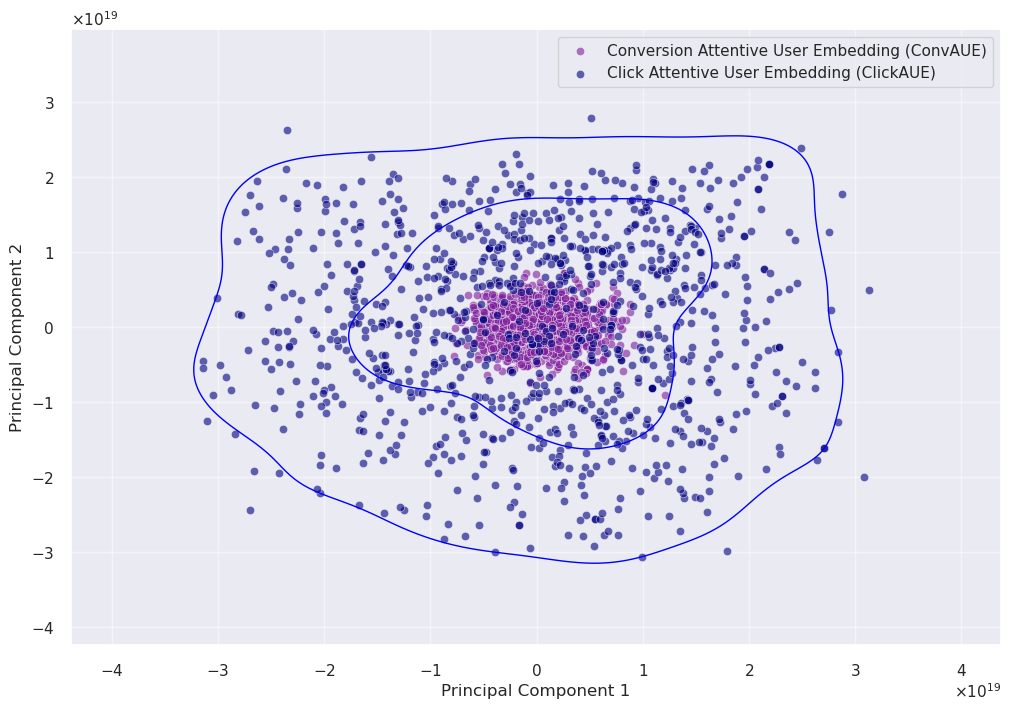} 
        \caption{Principal components of quantized user embeddings. ConvAUE and ClickAUE occupy distinct regions, indicating non-redundant information.}
        \label{fig:sub_a}
    \end{minipage}%
    \hfill 
    \begin{minipage}[t]{0.48\textwidth} 
        \centering
        \includegraphics[width=\linewidth]{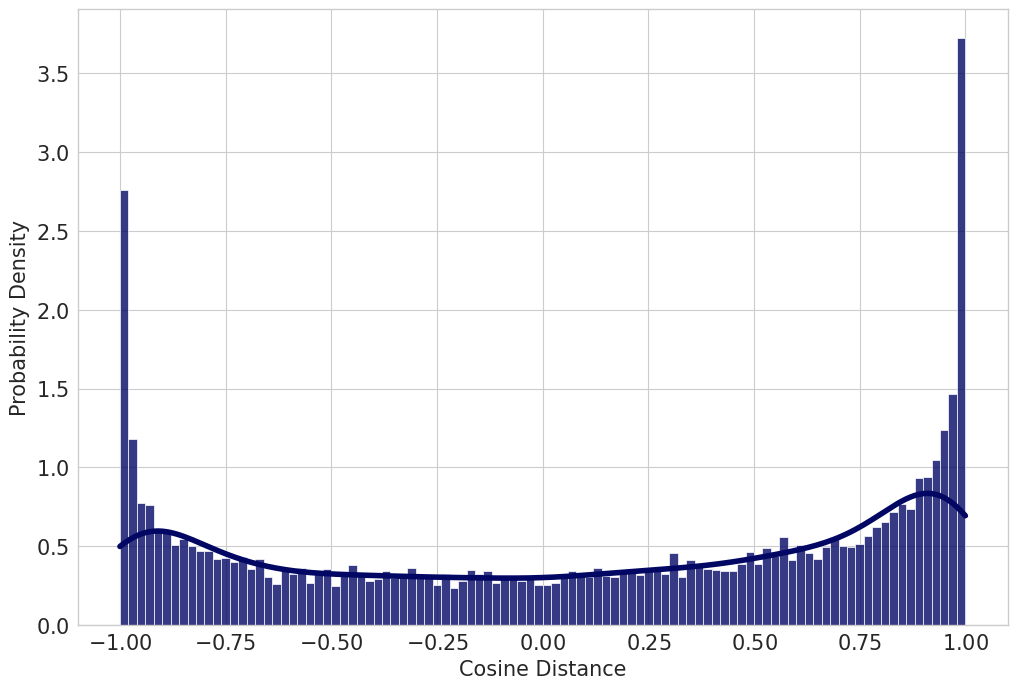}
        \caption{Orthogonality between ConvAUE and ClickAUE. Degree of complementarity varies, higher orthogonality indicating greater potential for additive information gain.}
        \label{fig:sub_b}
    \end{minipage}
    \label{fig:combined_figure}
\end{figure*}

\subsection{Embedding Analysis}        
  \label{sec:embedding_analysis}                                                                              
  We examine whether ClickAUE and ConvAUE encode redundant or complementary information through two analyses. 
  \paragraph{Principal component visualization.}                                                              
  Figure~\ref{fig:sub_a} projects quantized ClickAUE and ConvAUE vectors onto their first two principal components. The two embedding types occupy largely disjoint regions, indicating that they capture distinct aspects of user behavior rather than redundant representations of the same signal.
  \paragraph{Cosine distance distribution.}                 
  Figure~\ref{fig:sub_b} plots the probability density of pairwise cosine distances between ClickAUE and ConvAUE for the same users. The distribution is approximately uniform over $[-1, +1]$ with mild concentrations at both extremes. A uniform distribution implies that the two embedding spaces are approximately orthogonal in aggregate --- neither systematically aligned nor systematically opposed. The slight peaks near $-1$ and $+1$ suggest that a small subset of users exhibit strongly correlated or anti-correlated click and conversion patterns, while the majority of users are represented by embeddings that carry independent information. This near-orthogonality is consistent with the additive NE gains observed when combining both embeddings (Section~\ref{sec:main_results}): the two representations
  contribute largely non-overlapping predictive signal to the downstream ranker.

\subsection{Main Results}
\label{sec:main_results}

  Table~\ref{tab:prodranker_ne} reports the relative training NE reduction on the primary OCVR task. ConvAUE alone yields a 0.3\% gain, ClickAUE alone yields 0.21\%, and DUET (both embeddings combined) achieves 0.38\%. The theoretical upper bound --- computed under the assumption that ClickAUE and ConvAUE are perfectly orthogonal and their gains are fully additive --- is 0.51\%. DUET recovers approximately 75\% of this upper bound (0.38/0.51) with at least $\approx13\%$ relative lift compared to standalone embeddings, indicating complementarity between the two embeddings while suggesting partial information overlap, consistent with the cosine distance analysis in Section~\ref{sec:embedding_analysis}. Notably, neither single-stream embedding alone approaches the combined gain, confirming that the click and conversion streams encode distinct predictive signals that are jointly more informative than either in isolation.

\begin{table}
  \centering
  \caption{Training NE gain of downstream rankers, DUET outperforms and approaches theoretical upperbound.}
  \label{tab:prodranker_ne}
  \begin{tabular}{l c}
    \toprule
    \multicolumn{1}{c}{Downstream Ranker} & \%$\Delta$NE Gain\\
    \hline
    With ClickAUE & 0.21\% \\
    With ConvAUE & 0.30\% \\
    With DUET & 0.38\% \\
    \hline
    Upperbound: ConvAUE $\perp$ ClickAUE & 0.51\% \\
    \bottomrule
  \end{tabular}
\end{table}




    \begin{figure}[htbp] 
    \centering 
    \includegraphics[scale=0.65]{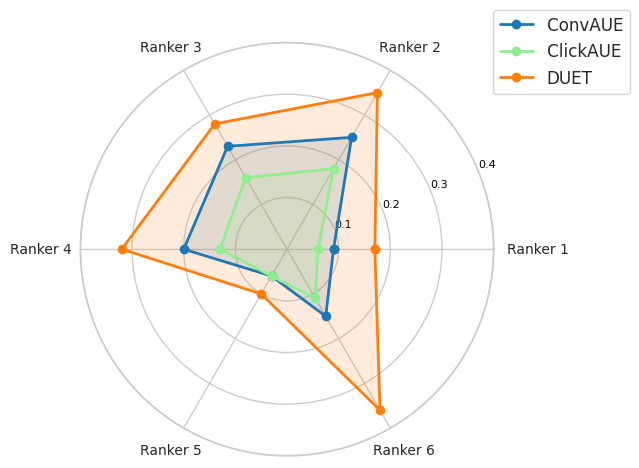} 
    \centering 
    \caption{Eval NE gains (\%) on Different Types and Stages of downstream rankers. Higher is better.} 
    \label{fig:prod_ne_radar} 
    \end{figure}
    
  Figure~\ref{fig:prod_ne_radar} reports evaluation NE gains across six downstream offsite CVR models. These are distinct OCVR ranking models, spanning different ranking stages (e.g. final-stage feed vs. early-stage explore) and optimization objectives. They differ in traffic volume, attribution characteristics, and baseline feature sets. We omit further architectural and operational details as they are not important for the evaluation; the important property is that all six share the same DLRM backbone and differ only in the input feature set when augmented with DUET embeddings.

  Three observations emerge. First, DUET consistently outperforms ConvAUE/ClickAUE across all six models, with additive lifts from ClickAUE ranging from +0.04\% (Ranker 5) to +0.21\% (Ranker 6). This confirms that click-derived representations provide complementary signal to conversion-derived representations regardless of the downstream model configuration. Second, the magnitude of both ConvAUE/ClickAUE and DUET gains varies across models --- Ranker 2 and Ranker 6 benefit most (0.35\% and 0.36\% respectively), while Ranker 5 shows more modest improvement (0.10\%). This variation reflects differences in data volume, attribution rates, and the degree to which each model's existing feature set already captures click- and conversion-related patterns. Third, the relative contribution of ClickAUE (the gap between DUET and ConvAUE) is not constant: Ranker 6 shows the largest additive lift (+0.21\%), suggesting that its baseline feature set has the most room for enrichment from click-stream representations, whereas Ranker 3 and Ranker 5 show smaller incremental gains, indicating that their existing features already capture some of the information encoded in ClickAUE.
  
\subsection{Ablation Study}
\label{sec:ablation}
We perform an ablation study on how individual components—training data, model architecture complexity, and event features—affect pre-training (PT) and downstream ranker (DR) performance. Table~\ref{tab:ablation} reports PT and DR NE degradation when each component is removed from its upstream encoder; larger values indicate greater importance.

ConvAUE ablations. The user journey event feature is the most impactful single component across both encoders, with removal incurring a 0.30\% PT NE loss and 0.1\% DR NE loss. This feature captures cross-site browsing patterns that provide direct evidence of conversion intent. Unattributed synthetic data contributes 0.13\% PT NE and 0.03\% DR NE, confirming that conversion events lacking deterministic attribution nonetheless provide useful weak supervision for the sparse OCVR stream. Architectural depth (2-layer vs.\ 1-layer attention) accounts for 0.06\% PT and 0.025\% DR, a modest but consistent gain that validates the interleaved cross-/self-attention design under sparse supervision.

ClickAUE ablations. Scaling the sequence and non-sequence architecture yields up to 0.13\% PT NE and 0.05\% DR NE gain, indicating that model capacity in the dense-feature branch is load-bearing for ClickAUE quality. Among individual event features, page engagement is the most impactful (0.07\% PT and 0.03\% DR), followed by target item impression events (0.04\% PT and 0.015\% DR).

Cross-encoder comparison. ConvAUE is dominated by a single high-signal source (user journey at 0.30\%), whereas ClickAUE draws gains uniformly with no event feature exceeding 0.07\%. This reflects the data regimes: the sparse conversion stream depends on few high-signal events, while the dense click stream benefits from breadth across many modest sources.

PT-to-DR transfer ratio. Comparing PT and DR NE changes reveals that sequence architecture modifications and event feature additions exhibit transfer ratios of 30\%–40\%, as these changes directly affect the sequence learning component used in embedding generation. Data-level changes such as synthetic data show a lower but still meaningful transfer ratio (23\%), confirming that upstream data improvements propagate to downstream rankers through the learned embedding representations.

\begin{table}[t]
  \centering
  \caption{Ablation study on the upstream encoders.}
  \label{tab:ablation}
  \begin{tabular}{l c c}
    \toprule
    \multicolumn{1}{c}{\textbf{ConvAUE}} & PT \%$\Delta$NE & DR \%$\Delta$NE \\
    \hline
    Remove Synthetic Data & 0.13\% & 0.03\%\\
    Reduce 1 Layer Attention & 0.06\% & 0.025\%\\ 
    Remove User Journey Event  & 0.3\% &0.1\% \\ 
    Remove Item Impression Event & 0.04\% & 0.015\%\\
    \midrule
    \multicolumn{1}{c}{\textbf{ClickAUE}} & & \\
    \hline
    Scaling Sequence/Non-sequence Arch &  0.13\% & 0.05\%\\
    Ablate Item Impression Event Feature & 0.04\% & 0.01\%\\
    Ablate Page Event Feature & 0.07\% & 0.03\%\\
    \bottomrule
  \end{tabular}
\end{table}

\subsection{A/B Test Results}
\label{sec:online}

\begin{table}
  \centering
  \caption{A/B test results with relative lifts vs.\ the baseline.}
  \label{tab:online}
  \begin{tabular}{l ccc}
    \toprule
    Downstream Ranker & Offsite CVR Lift \\
    \midrule
    Ranker 3 & +0.66\% \\
    Ranker 4 & +0.15\% \\
    \bottomrule
  \end{tabular}
\end{table}


  We validate DUET's impact through A/B tests on two CVR models (Table~\ref{tab:online}). Experimentation in recommendation systems is resource-intensive and this constrains the number of models that can be evaluated concurrently. 

  Both models show statistically significant improvements. Ranker 3 achieves a +0.66\% CVR lift; Ranker 4 achieves a +0.15\% CVR lift. Therefore, both models deliver positive lifts, confirming that the offline NE gains (Section~\ref{sec:main_results}) translate to measurable A/B test impact. All results are statistically significant at $p < 0.05$ (two-sided $t$-test). Latency overhead is negligible, as embeddings are pre-computed via the ETI system (Section~\ref{sec:eti}) and served through feature store lookup with no additional model inference at serving time.  
  

\section{Conclusion}
\label{sec:conclusion}

 We presented DUET, a framework that partitions heterogeneous behavioral data into domain-coherent click and conversion streams, pre-trains dedicated upstream encoders with stream-matched attention architectures, and integrates the resulting complementary embeddings into downstream rankers via asynchronous serving. The contribution lies not in the individual components --- which draw on established techniques --- but in their principled composition for offsite CVR prediction and detailed validation. Our experiments demonstrate as much as $0.38$\% NE reduction over the strongest baseline, and A/B tests across multiple downstream models confirm consistent CVR lifts with negligible latency overhead.

  \paragraph{Limitations and Future Work.}
  The current data routing is static and rule-based; learned routing mechanisms could enable finer-grained, multi-stream partitioning. The attention configuration assignment (self-attention for clicks, cross-attention for conversions) was validated through ablation but not compared against alternative assignments. Frozen embeddings during downstream training favor operational decoupling at the potential burden of adaptation to distribution shifts --- lightweight fine-tuning strategies (e.g., adapter layers) could mitigate this tradeoff.

\bibliographystyle{ACM-Reference-Format}
\bibliography{paper}
\newpage
\section{Appendix}
  
\subsection{Upstream Encoder Details}\label{sec:appendix_upstream_encoder}

We optimize the models using Distributed Shampoo optimizer~\cite{Gupta2018Shampoo}, a second-order method that leverages Kronecker-factored preconditioning to capture pairwise gradient correlations. Compared to first-order optimizers such as Adam, AdaGrad, and SGD, Shampoo consistently yields statistically significant improvements while incurring no additional inference overhead. We configure Shampoo with learning rate $\alpha=0.04$, $\beta_1=0.9$, $\beta_2=1.0$, $\epsilon=10^{-4}$, momentum $\mu=0$, and weight decay $\lambda=10^{-5}$. We apply linear learning rate warmup over $20,000$ iterations, interpolating from the initial rate to $10^{-3}$. 

\subsection{Equivalence of Consolidated and Separate Upstream Models}
\label{sec:appendix_proof}

In this section, we argue that the impact of AFL embeddings in the downstream model is approximately the same regardless of whether they are learned by one consolidated or two separate upstream models.

\begin{assumption} \label{assum:same_arch}
The AFL upstream model matches the baseline architecture, except it includes the sequence learning component, while the baseline uses a fixed embedding (i.e., the sequence component is frozen in the baseline).
\end{assumption}

\begin{assumption} \label{assum:same_interaction}
In a consolidated model, interactions between different data sources or intermediate embeddings occur within the upstream model; with separate upstream models, their interaction is deferred to the downstream model (e.g., within the interaction layer or overarch).
\end{assumption}

\begin{assumption} \label{assum:dot_prod}
Interactions between embedding vectors are realized through dot product operations. Note that this is a simplifying assumption; in practice, DLRMs use nonlinear interaction layers (e.g., DCN cross layers, MLPs), which may introduce deviations from the equivalence derived below.
\end{assumption}

\begin{lemma}
Learning a joint user embedding with a consolidated model utilizing both click and conversion data is nearly equivalent to separately learning embeddings from each data type in the upstream and integrating them together in the downstream model.
\end{lemma}

\begin{proof}
Let $\mathbf{u} \in \mathbb{R}^n$ be ConvAUE and $\mathbf{v} \in \mathbb{R}^n$ be ClickAUE.
Let $\mathbf{w} \in \mathbb{R}^n$ be scalar weights in the consolidated upstream, $\mathbf{w'}, \mathbf{w''} \in \mathbb{R}^n$ be scalar weights in ClickAUN and ConvAUN respectively, and $\mathbf{z} \in \mathbb{R}^n$ be a scalar vector in the downstream model.
By Assumption~\ref{assum:same_interaction}, the downstream model learns only $\mathbf{z} \in \mathbb{R}^n$ regardless of whether the upstream is consolidated or separate.

Table~\ref{table:afl_embeddings} summarizes the embeddings and associated weights in both settings.

\begin{table}[h]
    \centering
    \begin{tabular}{lcc}
        \toprule
         & Upstream & Downstream \\
        \midrule
        Consolidated & $(\mathbf{u} \cdot \mathbf{v}) \mathbf{w}$  & $\big((\mathbf{u} \cdot \mathbf{v}) \mathbf{w}\big)\mathbf{z}$ \\
        Separate & $\mathbf{u} \mathbf{w'},\; \mathbf{v} \mathbf{w''}$ & $\big(\mathbf{u} \mathbf{w'}\cdot \mathbf{v} \mathbf{w''}\big)\mathbf{z}$  \\
        \bottomrule
    \end{tabular}
    \caption{Consolidated vs.\ independent user embeddings in upstream and downstream settings.}
    \label{table:afl_embeddings}
\end{table}

For the consolidated downstream representation:
\begin{align} \label{eq:consolidated_down}
    (\mathbf{u} \cdot \mathbf{v}) \mathbf{w} &= \sum_{i=1}^n w_i (u_i \cdot \mathbf{v}) \notag\\
    &= (w_1 u_1) \cdot \mathbf{v} + \cdots + (w_n u_n) \cdot \mathbf{v} && \text{(homogeneity)} \notag\\
    &= \bigg(\sum_{i=1}^n w_i u_i\bigg) \cdot \mathbf{v} && \text{(distributivity)}
\end{align}

For the separate downstream representation, using Assumption~\ref{assum:same_interaction}:
\begin{align} \label{eq:separate_down}
    \mathbf{u} \mathbf{w'} \cdot \mathbf{v} \mathbf{w''} &= (\mathbf{u} \circ \mathbf{w'}) \cdot (\mathbf{v} \circ \mathbf{w''}) && \text{(Hadamard product)} \notag\\
    &= \sum_{i=1}^n u_i w'_i v_i w''_i \notag\\
    &= \sum_{i=1}^n w^*_i u_i v_i \notag\\
    &= \bigg(\sum_{i=1}^n w^*_i u_i\bigg) \cdot \mathbf{v}
\end{align}

By Assumption~\ref{assum:same_arch}, comparing Equations~\eqref{eq:consolidated_down} and~\eqref{eq:separate_down}, we have $\mathbf{w^*} \approx \mathbf{w}$, establishing near-equivalence.
\end{proof}

\paragraph{Discussion.}
The near-equivalence result shows that separating embeddings into two upstream models does not sacrifice representational power relative to a consolidated model, under linear interaction assumptions.
However, the separate design offers important advantages beyond equivalence: it enables each upstream model to use a \emph{different} attention architecture tailored to its data domain (pure self-attention for ClickAUM vs.\ interleaved cross-attention and self-attention for ConvAUM), which a single consolidated model cannot easily accommodate.
Furthermore, separate models provide operational benefits including independent scaling, fault isolation, and the ability to retrain each pipeline on its own schedule without affecting the other.
We note that in practice, DLRMs use nonlinear interaction layers (e.g., DCN cross layers, MLPs) rather than pure dot products, which means the actual equivalence may be looser than the linear case derived above; nonetheless, our empirical results (Section~\ref{sec:main_results}) confirm that the separate design performs at least as well as single-stream pre-training while offering the architectural flexibility described above.

\end{document}